\documentclass{svproc}
\usepackage{url,graphicx, dsfont}

\usepackage{amsmath,amssymb,amsfonts,xspace,epsfig,stackrel,latexsym,wrapfig,verbatim,psfrag,rotating}
\usepackage[top=1in, bottom=1in, left=1in, right=1in]{geometry}
\RequirePackage[colorlinks,citecolor=blue,urlcolor=blue]{hyperref}

\usepackage{subfigure}

\newcommand{\twopiece}[5][0]{% use option [1] if display is desired; default [0] is not in display
    \ifcase#1
        \left\{\begin{array}{ll}{#2}&{\text{ if } #3}\\{#4}&{\text{ if } #5}\end{array}\right.
    \else
        \left\{\begin{array}{ll}{#2}&{\text{ if } #3}\vspace{#1pt}\\{#4}&{\text{ if } #5}\end{array}\right.
\fi}

\newcommand{\threepiece}[7][0]{ %use option [1] to add space (number of pts) between lines; default [0] is not in display
    \ifcase#1
        \left\{\begin{array}{ll}{#2}&{\text{ if } #3}\\{#4}&{\text{ if } #5}\\{#6}&{\text{ if } #7}\end{array}\right.
    \else
        \left\{\begin{array}{ll}{#2}&{\text{ if } #3}\vspace{#1pt}\\{#4}&{\text{ if } #5}\vspace{#1pt}\\{#6}&{\text{ if } #7}\end{array}\right.
\fi}

\def\R{\mathds{R}}

\newcommand{\eps}{\varepsilon}

\newcommand{\sign}{\operatorname{sign}}

\begin{document}

\mainmatter              % start of a contribution
\title{Data-aware customization of activation functions reduces neural network error}
\titlerunning{Data-aware customization of activation}  % abbreviated title (for running head)
\author{Fuchang Gao\inst{1} \and Boyu Zhang\inst{2}}
\authorrunning{Fuchang Gao and Boyu Zhang} % abbreviated author list (for running head)
\tocauthor{Fuchang Gao, Boyu Zhang}
\institute{Department Mathematics and Statistical Science, University of Idaho\\ 875 Perimeter Drive, MS 1103, Moscow, ID 83844-1103, USA\\
\email{fuchang@uidaho.edu},\\ WWW home page:
\texttt{https://www.webpages.uidaho.edu/~fuchang/}
\and Institute for Modeling Collaboration and Innovation, University of Idaho\\
        875 Perimeter Dr MS 1122, Moscow, ID 83844-1122, USA
}

\maketitle              % typeset the title of the contribution

\begin{abstract}
Activation functions play critical roles in neural networks, yet current off-the-shelf neural networks pay little attention to the specific choice of activation functions used. Here we show that data-aware customization of activation functions can result in striking reductions in neural network error. We first give a simple linear algebraic explanation of the role of activation functions in neural networks; then, through connection with the Diaconis-Shahshahani Approximation Theorem, we propose a set of criteria for good activation functions. As a case study, we consider regression tasks with a partially exchangeable target function, \emph{i.e.} $f(u,v,w)=f(v,u,w)$ for $u,v\in \mathbb{R}^d$ and $w\in \mathbb{R}^k$, and prove that for such a target function, using an even activation function in at least one of the layers guarantees that the prediction preserves partial exchangeability for best performance. Since even activation functions are seldom used in practice, we designed the ``seagull'' even activation function $\log(1+x^2)$ according to our criteria. Empirical testing on over two dozen 9-25 dimensional examples with different local smoothness, curvature, and degree of exchangeability revealed that a simple substitution with the ``seagull'' activation function in an already-refined neural network can lead to an order-of-magnitude reduction in error. This improvement was most pronounced when the activation function substitution was applied to the layer in which the exchangeable variables are connected for the first time. While the improvement is greatest for low-dimensional data, experiments on the CIFAR10 image classification dataset showed that use of ``seagull'' can reduce error even for high-dimensional cases. These results collectively highlight the potential of customizing activation functions as a general approach to improve neural network performance.

\end{abstract}

\section{Introduction}
The last decade has witnessed the remarkable success of deep neural networks in analyzing high dimensional data, especially in computer vision and natural language processing. In recent years, these successes have sparked significant interest in applying deep neural networks to make scientific discoveries from experimental data, including the identification of hidden physical principles and governing equations across a wide range of problem settings \cite{carleo2019machine,iten2020discovering,raissi2018hidden,de2019discovery}.

While discoveries of scientific principles can range from conceptual and qualitative to precise and quantitative, one common desire is to discover interpretatable and quantitative relationships between target and input variables. To achieve such an ambitious goal, one must integrate existing scientific knowledge into neural network design. Despite the efforts aimed at improving the interpretability of deep neural networks, they are still largely viewed as ``black boxes." This ``black box'' nature sometimes is convenient because it allows users to create useful models and obtain results without requiring a complete understanding of its inner workings. However, these models tend to generalize poorly across different application settings, making it difficult for users to re-design the model for new applications of interest.

In real-world applications, data sets are often associated with extra information that is very important to domain experts---for example, underlying constraints between variables, or meta-data about which variables (measurements) are more or less reliable than others. Inferring this extra information \emph{de novo} from the data may be difficult for various reasons, such as computational cost, sampling bias, or limitations of the model itself. For instance, invariance to reflection across a vertical axis is essential knowledge for object classification from images, but this information is difficult to learn from limited training data.

At the same time, building this type of knowledge into the model is expected to increase the model's performance for the specific data set at hand. Currently, this is commonly done by data augmentation or by introducing additional terms in the loss function. In this paper, we emphasize that analytically architecting a neural network
based on underlying data structure and customizing activation functions provide a more efficient approach than data augmentation or  customizing loss functions.

This view is shared by many researchers. For example, \cite{wang2020incorporating} investigated tailoring deep neural network structures to enforce a different symmetry. It is not a surprise that their methods improved the generalization of the model. On the other hand, by purposefully designing neural network architecture and components and testing their effectiveness on data, we may let machine learning models guide us to discover new scientific laws hidden in the data.

In this paper, we first give a linear algebraic explanation of the inner workings of a neural network.
This simple and transparent explanation helps us to get a better understanding of the role of activation functions in neural networks. Then, through the connection with the Diaconis-Shahshahani Approximation Theorem, we discuss the importance and criteria for customizing activation functions. As a case study, we consider one special yet common case when the target function $f$ is partially exchangeable, \emph{i.e.} $f(u,v,w)=f(v,u,w)$ for $u,v\in \mathbb{R}^d$, and $w\in \mathbb{R}^k$. We proved that for such a target function, using an even activation function in at least one of the layers guarantees that the prediction preserves partial exchangeability for best performance. Since even activation functions are seldom used in practice, we then designed a ``seagull" activation function $\log(1+x^2)$ based on our criteria of good activation, and empirically tested over two dozen 9-64 dimensional regression problems, revealing a significant improvement. Improvement was observed regardless of the original activation function, the loss function, the optimizer, potential noise in the data, but it was most pronounced when the activation function substitute was applied to the layer in which the exchangeable variables are connected for the first time. While the improvement is greatest for low-dimensional data, our experiments on CIFAR10 image classification problem showed that even for high-dimensional case with low degree of exchangeability, the improvement can still be noticeable.

\section{A Linear Algebraic Explanation of the Role of Activation}
For a deep neural network with $k$-hidden layers, the approximation can be written as:
\begin{align*}
f(x_1,x_2,\ldots, x_d)
\approx \sum_{i=1}^{m_k} g\left(\phi_{i}^{k}(x_1,x_2,\ldots, x_d)\right)w_{i}^k+b^k,
\end{align*}
where $\phi_{i}^{k}$ are recursively defined by $\phi_{i}^{1}(x_1,x_2,\ldots, x_d)=x_1w_{1i}^{1}+x_2w_{2i}^1+\cdots+x_dw_{di}^1+b_i^{1}$, and
\begin{align*}
\phi_{i}^{r}(x_1,x_2,\ldots, x_d)
=\sum_{j=1}^{m_r}g\left(\phi_j^{r-1}(x_1,x_2,\ldots, x_d)\right)w_{ji}^{r}+b_i^r,\ \ r=2,3,\ldots, k.
\end{align*}

It is well known that neural networks with at least one hidden layer can approximate any multivariate continuous function (\cite{gybenko1989approximation} \cite{hornik1991approximation} \cite{ismailov2014approximation}).
While the proofs are not difficult, they are usually abstract. For example, the proof in \cite{gybenko1989approximation} used Hahn-Banach Theorem. To better understand the inner workings of neural networks, here we use linear algebra to explain the approximation mechanism. For clarity, we only consider neural networks with one hidden layer and assume the error is measured by the mean square error.

For each input data ${\bf x}_j=(x_{j1}, x_{j2}, \ldots, x_{jd})$, the output
$$\sum_{i=1}^m g\left(x_{j1}w_{1i}+x_{j2}w_{2i}+\cdots+x_{jd}w_{di}+b_i\right)\alpha_i+\beta$$
is meant to approximate the target value $y_j$. If we denote
\begin{align}
X = \left(
        \begin{array}{cccc}
          x_{11} & x_{12} & \cdots & x_{1d} \\
          x_{21} & x_{22} & \cdots & x_{2d} \\
          \vdots & \vdots & \cdots & \vdots \\
          x_{N1} & x_{N2} & \cdots & x_{Nd} \\
        \end{array}
      \right),
\ \
W = \left(
        \begin{array}{cccc}
          w_{11} & w_{12} & \cdots & w_{1m} \\
          w_{21} & w_{22} & \cdots & w_{2m} \\
          \vdots & \vdots & \cdots & \vdots \\
          w_{d1} & w_{d2} & \cdots & w_{dm} \\
        \end{array}
      \right),
\end{align}
and denote $y=(y_1, y_2, \ldots, y_N)^T$, $b=(b_1,b_2,\ldots, b_m)$, $\alpha=(\alpha_1,\alpha_2,\ldots, \alpha_m)^T$, and let $\mathds{1}$ be the $N$-dimensional column vector whose entries are all 1, then the problem can be formulated as finding the least square solution to the problem
$$g(XW+\mathds{1}b)\alpha+\mathds{1}\beta= y,$$
where $g(U)$ means to apply the function $g$ to every entry of $U$.

When $g(XW+\mathds{1}b)$ is fixed, there is an analytical solution
\begin{align}
\alpha =& \left(g(XW+\mathds{1}b)-\overline{g(XW+\mathds{1}b)}\right)^+y, \\
 \beta =& \overline{y}-\overline{g(XW+\mathds{1}b)}\alpha,
\end{align}
where $U^+$ means the Moore–Penrose inverse of the rectangular matrix $U$, and $\overline{M}$ means the row mean of the matrix $M$.
Thus, the main goal of the problem is to find $W$ and $b$ so that the affine span of $g(XW+\mathds{1}b)$ can best approximate the $N$-dimensional vector $y$.

Note that the matrix $XW+\mathds{1}b$ has rank at most $d+1$. The role of the activation function is to create a matrix $g(XW+\mathds{1}b)$ with a much higher rank, so that affine combinations of its column vectors can approximate the $N$-dimensional vector $y$. Not every function $g$ has the such a potential. For example, if $g(t)$ is a polynomial of degree $p$, then the matrix $g(XW+\mathds{1}b)$ has rank at most ${d+p}\choose{p}$, regardless of the size of $W$. Fortunately, the following theorem says that polynomials are essentially the only functions that have this limitation.

\begin{theorem}
Suppose $g(t)$ is not a polynomial in some open interval $I\subset \R$, in which there exists one point at which $g$ is infinitely many times differentiable. For any distinct points ${\bf x}_i$, $1\le i\le N$ in $\R^d$, with $N\gg d$, and any $d<m\le N$, there exists a $d\times m$ matrix $W$ and an $m$-dimensional row vector $b$, such that $g(XW+\mathds{1}b)$ has rank $m$, where $X$ is the $N\times d$ matrix whose $i$-th row is the vector ${\bf x}_i$. The same statement holds if $g(t)=\max\{t,0\}$. On the other hand, if $g$ is a polynomial of degree $p$, then the rank of $g(XW+\mathds{1}b)$ is at most ${p+d}\choose{p}$ regardless of the size of $W$.
\end{theorem}

\proof
Because the $N$ points are distinct, there is a direction $(w_1,w_2,\ldots, w_d)^T$ so that the scalar projections of the points on this vector have distinct values. That is, each term in the  following sequence
$$c_k:=x_{k1}w_{1}+x_{k2}w_{2}+\cdots+x_{kd}w_{d}, \ k=1,2,\ldots, N$$
has a distinct non-zero value.

Let us first consider the case where $g$ is infinitely many times differentiable on $I$ and is not a polynomial. In this case, there exists $b_0\in I$ such that the derivatives $g^{(k)}(b_0)\ne 0$ for $0\le k\le N$. We show that for any $d<m\le N$, we can choose a $d\times m$ matrix $W=(w_1,w_2,\ldots, w_d)^T(t_1,t_2,\ldots, t_m)$ and $b=(b_0,b_0,\ldots, b_0)$ such that the matrix $g(XW+\mathds{1}b)$ has rank $m$. Suppose this is not true. Let $S$ be the linear span of first $m-1$ columns of $g(XW+\mathds{1}b)$, and consider the matrix $W(t)=(w_1,w_2,\ldots, w_d)^T(t_1,t_2,\ldots, t_{m-1},t)$. Since for all $t$, the last column of the matrix $g(XW(t)+\mathds{1}b)$, that is, the vector $v(t):=(g(c_1t+b_0), g(c_2t+b_0),\ldots, g(c_Nt+b_0))^T$ belongs to the space $S$, which is independent of $t$. So does the derivatives of $v(t)$. Taking derivative of $v(t)$ with respect to $t$ for $k$-times, $k=0,1,...,m-1$ and evaluate it at $t=0$, we obtain $v^{(k)}(0)= g^{(k)}(b_0)(c_1^k,c_2^k,\ldots, c_N^k)^T$. Since $g^{(k)}(b_0)\ne 0$ for all $k=0,1,2,\ldots m-1$,  We have $(c_1^k, c_2^k, \ldots, c_N^k)^T\in S$, $0\le k\le m-1$. This is a contradiction because the $m$ vectors $(c_1^k, c_2^k, \ldots, c_N^k)$, $0\le k\le m-1$ are linearly independent, while the dimension of $S$ is at most $m-1$. Thus, $g(XW+\mathds{1}b)$ has rank $m$ for some $d\times m$ matrix $W$ and some $m$-dimensional row vector $b$.

Now, we consider the case when $g(t)=\max\{t,0\}$. Without loss of generality, we assume $c_1>c_2>\cdots>c_N$. By choosing $W=(w_1,w_2,\ldots, w_d)^T(1,1,\ldots, 1)$ and letting $b=(-s_1,-s_2,\ldots, -s_m)$ where $c_k>s_k\ge c_{k+1}$, $1\le k\le m$, we see that the $k$-th column of $g(XW+\mathds{1}b)$ is the vector $(c_1-s_k,c_2-s_k,\ldots, c_k-s_k, 0,0,\ldots, 0)^T$ with $k$-th coordinate non-zero, all the coordinates after the $k$-th equal 0. It is clear that the linear span of such column vectors has dimension $m$.

In the other direction, suppose $g$ is a polynomial of degree $p$, then all the column vectors of $g(XW+\mathds{1}b)$ can be expressed as a linear combination of entry-wise products of the form $u_{j_1}u_{j_2}\cdots u_{j_k}$ for $1\le k\le p$, where $u_j =(x_{1j}, x_{2j}, \ldots, x_{Nj})^T$. Thus, the rank of $g(XW+\mathds{1}b)$ is at most ${d+p}\choose{p}$.

\begin{remark}
It is known that a continuous non-polynomial activation function ensures uniform approximability on a compact domain (Theorem 3.1 in \cite{Pinkus}). Here we only gave a quick linear algebraic proof under a stronger assumption of infinite differentiability at one point, paving the road to constructing a new activation function in the next section. Also, the result about ReLU function is known (e.g. Theorem 1 in \cite{Zhang}).
\end{remark}

\section{Customizing Activation Functions}
While all functions except polynomials can be a valid choice for an activation function, a good activation may significantly reduce the number of combination terms needed. Thus, customizing activation functions can be very important. The following theorem of Diaconis and Shahshahani \cite{diaconis1984nonlinear} can be viewed as the first promotion of customizing activation functions: Any multivariate continuous function $f$ on a bounded closed region in $\mathbb{R}^d$ can be approximated by:
\begin{align} \label{eq1}
f(x_1,x_2,\ldots, x_d)\approx \sum_{i=1}^m g_i\left(a_{i1}x_1+a_{i2}x_2+\ldots+a_{id}x_d\right),
\end{align}
where $g_i$ are non-linear continuous functions (\cite{diaconis1984nonlinear} Theorem 4). Thus, choosing good activation functions $g_i$ are very important for accurate approximation.

The current off-the-shelf neural networks pay little attention to choosing activation functions.
Instead, they simply use one of a few fixed activation functions such as ReLU (i.e., $\max(x,0)$).
While such a practice seems to work well especially for classification problems, it may not be the most efficient way for regression problems.
This greatly limits the potential of the neural networks.

A good choice of activation functions can greatly increase the efficiency of a neural network.
For example, to approximate the function $z=\sin(xy)$, a neural network may use the following step-by-step strategy (\ref{eq3}) to bring $(x,y)^T$ to $\sin(xy)$:
\begin{equation} \label{eq3}
\begin{split}
\left(
    \begin{array}{c}
      x \\
      y \\
    \end{array}
 \right)
 & \stackrel{\rm{linear}}{\longrightarrow}
	\left(
    \begin{array}{c}
      x+y \\
      x-y \\
    \end{array}
  \right)
  \stackrel{(\cdot)^2}{\longrightarrow} \left(
    \begin{array}{c}
      (x+y)^2 \\
      (x-y)^2 \\
    \end{array}
  \right) \stackrel{\rm{linear}}{\longrightarrow}  xy \quad \stackrel{\sin(\cdot)}{\longrightarrow}  \sin(xy)
\end{split}
\end{equation}
In other words, if one happens to use the non-linear functions $f(t)=t^2$ and $g(s)=\sin(s)$ as the activation function in the first layer and the second layer, respectively, then even the exact function can be discovered. Of course, without knowing the closed form expression of $f$, it is impossible to choose the exact activation functions as exemplified above. However, with some partial information that likely exists from domain knowledge, one may design better activation functions accordingly. This is the goal of the current paper. Let us remark that this analytical-design approach is different from the adaptive activation functions approach used in Japtap et al. \cite{jagtap2019locally}, in which both layer-wise and neuron-wise
locally adaptive activation functions were introduced into Physics-informed
Neural Networks \cite{raissi2019physics}, and showed that such adaptive activation functions lead to improvements on training speed and accuracy on a group of benchmark data sets.

To customize activation functions, the main criteria is their potential to easily generate vectors outside the linear span of the input vectors. We have analyzed above that as long as the activation function is not a polynomial, it can generate vectors whose affine combinations can approximate all target vectors. Other than that, the performances of an activation function often depends on the datasets and the model hyperparameters and network architecture. In image classification, ReLU and Leaky ReLU are popular choices. For recurrent neural networks, sigmoid and hyperbolic tangent are more commonly used. Here at the risk of drawing criticisms, we post three criteria:\\

\noindent {\bf 1. Smoothness}. The smoothness of the activation function is a natural assumption for regression problems in which the target function is smooth. Consider the regression problem of approximating an unknown multivariate continuous function $f(x_1,x_2, \ldots, x_d)$ on a bounded region using a neural network that uses a gradient-based optimizer. For a gradient-based algorithm to work well, the partial derivatives of the function $f$ must exist. Thus, it is necessary that the function $f$ is of bounded Lipschitz, i.e., there exists a constant $L$ such that for all $X = (x_1,x_2,\ldots, x_d)$ and $Y = (y_1,y_2,\ldots, y_d)$ in the bounded region, we have
\begin{equation} \label{eq4}
|f(X)-f(Y)| \le L\sqrt{\Sigma_{i=1}^d(x_i-y_i)^2 }.
\end{equation}
If we allow some errors $\eps$ in Hausdorff distance in approximating the level sets $K$ of the target function, then we are approximating the set $K+B_{d-1}(\eps)$, where $B_{d-1}(\eps)$ is the ball in $\R^{d-1}$ with radius $\eps$. It is known that the surface of $K+B_{d-1}(\eps)$ is of $C^{1,1}$ (but $C^2$) even if $K$ is $c^\infty$ smooth \cite{kiselman}. To approximate such a set, it is best to use shapes of the same degree of smoothness. Thus, the activation function is better to have a continuous first derivative. Extra degree of smoothness is not necessary. In other words, using functions like $\log(1+x^4)$ for extra smoothness near $x=0$ is not necessary.\\

\noindent {\bf 2. Local curvature}. Although from our proof of Theorem 1, we only need one point $x_0$ at which the activation function $g$ satisfies $g^{(k)}(x_0)\ne 0$ for all $k$, to release the burden of choosing the best $W$ and $b$ in $g(XW+\mathds{1}b)$ to match the point $x_0 =XW+\mathds{1}b$, we require that $g^{(k)}(x)\ne 0$ for all $k$ at many places of $x$ as possible. In other words, $g$ is not a piecewise polynomial. At first look, this sounds like contradicting with the common use of cubic splines in approximating smooth functions. The difference is that this activation function is used to transform the intermediate variables, not to locally approximate the function itself.\\

\noindent {\bf 3. Growth rate}. From point of view of reducing generalization error, a good activation function should not put too much weight on a few particular variables. An outlier in $X$, may cause the value $XW+\mathds{1}b$ to increase significantly. If $g'(t)$ is large when $|t|$ is large, then a gradient-based optimizer will likely to make big change on the corresponding parameter in the weight matrix $W$, making the neural network difficult to train, and a small perturbation on a few variables could significantly alter the prediction. Thus, if an activation function is used in several layers, the growth rate at infinite should not be faster than linear. There are several dozens of activation functions that are commonly used. Analyzing these functions, we notice that besides a few bounded functions such as Binary Step, Sigmoid and Hyperbolic Tangent, all the remaining activation functions (such as in ReLU, Leaky ReLU, PReLU, ELU, GELU, SELU, SiLU, Softplus, Mish) grow linearly as $x\to \infty$.

To take all three criteria into consideration, and to fill in the gap between bounded and linear growth, we propose to consider the following activation functions:
$g_1(x)=\log(1+|x|^\alpha)$, $g_2(x)=\sign(x)\log(1+|x|^\alpha)$ and the function $g_3(x)=\log(1+|x|^\alpha)$ if $x>0$; and $g_3(x)=0$ if $x\le 0$. For $1\le \alpha\le 2$, the function $g_1(x)$ behaves like $x^\alpha$ for positive $x$ near $0$, and behaves like $\alpha\log x$ for $|x|$ large. The function $g_2$ and $g_3$ are the modifications. In particular, introduce
the following two functions: $\log(1+x^2)$ and ${\rm sign}(x)\log(1+|x|)$. For easy referencing, we call $\log(1+x^2)$ as the Seagull activation function, as its graph looks like a flying seagull, and call ${\rm sign}(x)\log(1+|x|)$ as Linear Logarithmic Unit (LLU). We will demonstrate the effectiveness of Seagull activation later in the paper. Figure 3 of \cite{Gao} showed another example that activation functions with logarithmic growth performed noticeably better than bounded or linear growth activation functions.

\section{A Case Study: Partially Exchangeable Targets}
In this section, we consider the case where $f$ satisfies the relation
\begin{equation}\label{eq5}
f(u,v,w)=f(v,u,w)
\end{equation}
for $u=(x_1,x_2,\ldots, x_k),v=(x_{k+1},x_{k+2},\ldots, x_{2k})\in [-1,1]^k$ and $w=(x_{2k+1}, x_{2k+2},\ldots, x_d)\in [-1,1]^{d-2k}$. For convenience, we say such functions are partially exchangeable with respect to two subsets of variables.
These functions are very common in practice. For example, if the function $f$ depends on the distance between two points of observation in space, then as a multivariate function of the coordinates of these two points (with or without factors), $f$ is partially exchangeable. As another example, if a label of an image is invariant under left-right flipping, then it is partially exchangeable. Similarly, the rotational invariance of 3D structures \cite{thomas2018tensor} and  view point equivalent of human faces \cite{leibo2017view} could also be described using partially exchangeable feature. Indeed, if $x_1,x_2,\ldots, x_m$ is the usual vectorization of the pixels, denote $u=(x_1,x_2,\ldots, x_k)$, $v=(x_{m},x_{m-1},\ldots, x_{m-k})$, where $k=m/2$ if $m$ is even, and $k=(m-1)/2$ if $m$ is odd. Then the label of the image can be expressed as $f(u,v)$ if $m$ is even, and $f(u,v,x_{k+1})$ if $m$ is odd. In either case, the function is partially exchangeable with respect to $u$ and $v$, i.e. $f(u,v)=f(v,u)$ or $f(u,v,x_{k+1})=f(v,u,x_{k+1})$.

A special case is that $f$ also satisfies
\begin{equation} \label{eq6}
f(u,v,w)=f\left(\frac{u+v}{2}, \frac{u+v}2,w\right).
\end{equation}
The latter case is more restrictive and less interesting, and will be called the trivial case. An example of the trivial case is that a function $f$ depends on the middle point of $u$ and $v$ in space, not the actual location of each individual points.

We first prove the following:
\begin{theorem} \label{th1}
Suppose a multi-layer neural network is trained to predict an unknown bounded Lipschitz function $f(x)$ on a region in $\mathbb{R}^d$ that contains the origin as its interior.
Suppose the target function is partially exchangeable with respect to some two subsets of variables.
Let $\widehat{f}(x)=\psi\circ g(xW)$ be the neural network prediction of $f$, where $W$ is a $d\times m$ is the weight matrix of the first hidden layer with activation function $g$ and no bias.
If $\widehat{f}$ keeps the non-trivial partial exchangeability of $f$ and is non-trivial, then $\psi\circ g(\cdot)$ is an even function, which can be achieved when $g$ is an even function.
\end{theorem}

\begin{proof}
Let $\alpha_1, \alpha_2, \ldots, \alpha_k, \beta_1,\beta_2,\ldots, \beta_k, \gamma_1,\gamma_2,\ldots, \gamma_{d-2k}$ be the row vectors of $W$. Denote $u=(x_1,x_2,\ldots, x_k)$, $v=(x_{k+1},x_{k+2},\ldots, x_{2k})$ and $w=(x_{2k+1}, x_{2k+2},\ldots, x_d)$. We have
$$\widehat{f}(u,v,w)=\psi\circ g\left(\sum_{i=1}^k\alpha_ix_i+\sum_{i=1}^k\beta_ix_{i+k}+\sum_{j=1}^{d-2k}\gamma_j x_{2k+j}\right).$$
In particular, from these expressions we have
\begin{align*}
&\widehat{f}(u,-u, 0)=\psi\circ g\left(\sum_{i=1}^k(\alpha_i-\beta_i)x_i\right),\\
&\widehat{f}(-u,u, 0)=\psi\circ g\left(-\sum_{i=1}^k(\alpha_i-\beta_i)x_i\right).
\end{align*}
By the non-trivial partial exchangeability assumption on $\widehat{f}$, we have $\widehat{f}(u,-u, 0)=\widehat{f}(-u,u, 0)$. Thus,
$$\psi\circ g(\sum_{i=1}^k(\alpha_i-\beta_i)x_i)=\psi\circ g(-\sum_{i=1}^k(\alpha_i-\beta_i)x_i),$$
which implies that $\psi\circ g(\cdot)$ is an even function in $\R^m$.
\end{proof}

\begin{remark}
Since the multivariate function $\psi$ depends on later layers and is typically very complicated, an easy and effective way to ensure that $\psi\circ g(\cdot)$ be an even function is to choose an even activation function $g$ in the first layer where the exchangeably variables are connected. The evenness of $\psi\circ g$ can also be achieved by the evenness of $\psi$, which can be attained by using an even activation function in later layers. For a convolutional neural network, to capture the left-right flipping invariance, one can achieve this by using an even activation function in the first fully-connected layer. To capture local exchangeability, one may use an even activation in the middle.
\end{remark}

Using an even activation function is clearly not a common practice. Indeed, almost all the popular activation functions are not even. As such, we will use the Seagull activation function that we have constructed above, i.e., the function $\log(1+x^2)$. Note that Theorem~\ref{th1} does not suggest that one should use an even activation function for all the layers. Also, our experiments do not seem to indicate that one can get further benefit from extensive use of even activation functions. We suggest to use it in the layer where the exchangeable variables are integrated together for the first time.

\section{Empirical Evaluation}
This section presents the usages of the proposed customization of activation function on both synthetic and real-world data. We first consider lower dimensional cases where there is a strong exchangeability in the target function.

\subsection*{Experiment on different target functions and different original activations} Consider function $y=f(u,v,w)$ being the area of the triangle with vertices $u=(x_1,x_2,x_3)$, $v=(x_4,x_5, x_6)$, and $w=(x_7, x_8, x_9)$. Clearly, $f(x_1,x_2,...,x_9)$ satisfies $f(u,v,w)=f(v,u,w)$. In fact, we have the closed formula $f(u,v,w)=\frac12\sqrt{A^2+B^2+C^2}$, where
$$A=(x_4-x_1)(x_8-x_2)-(x_7-x_1)(x_5-x_2),$$
$$B=(x_4-x_1)(x_9-x_3)-(x_7-x_1)(x_6-x_3),$$
$$C=(x_5-x_2)(x_9-x_3)-(x_8-x_2)(x_6-x_3).$$
In short, $f$ is a square root of a positive fourth degree polynomial of 9 variables. To ensure that the improved performance was not due to the specific formula of the target function $f$, we also consider the following four target functions $\log(1+f)$, $e^f/100$, $\sin(f)$ and $Q(f)=\sqrt{\frac{f^2+3}{f+1}}$. Thus, these five functions cover different rates of growth, from logarithmic growth to exponential growth to wavy function.

To discover an approximate formula using a neural network, we randomly sampled 10,000 9-dimensional vectors from $[-2,2]^9$, representing the $(x,y,z)$-coordinates of three points $u, v, w$ in $\R^3$ respectively, and created labels using the formulas above. In the same way, we independently generated 2000 vectors and the corresponding labels to form a test set.

For the function $f$, we then built several fully-connected neural networks and selected the one with the best overall performance. The selected model was a fully-connected neural network (\emph{i.e.} Multilayer Perceptron) with four hidden layers.
The input layer had 9 nodes and the output layer had 1 node. Each hidden layer had 100 nodes. For other four functions, we repeated the same procedure, and found that the same neural network architecture is close to be the best. Thus, we used the same architecture for all five functions. Based on the selected neural network architecture, for each target function, we built five models using five popular activation functions (ReLU, ELU, sigmoid, tanh, softplus), and trained all the 25 models with the RMSProp optimizer for 500 epochs with batch size was set as 100. The learning rate are tuned for each model function, starting from $(0.001, 0.003, 0.005)$ which was halved every 100 epochs. The best one was used.

To demonstrate the effectiveness of using the customized activation function, for each of the 25 selected models we replaced the activation function in the first hidden layer by the Seagull activation function $\log(1+x^2)$, while leaving all the remaining parts of the neural networks and hyperparameters unchanged.
\begin{figure}[h]
\centering
    \includegraphics[width=3in]{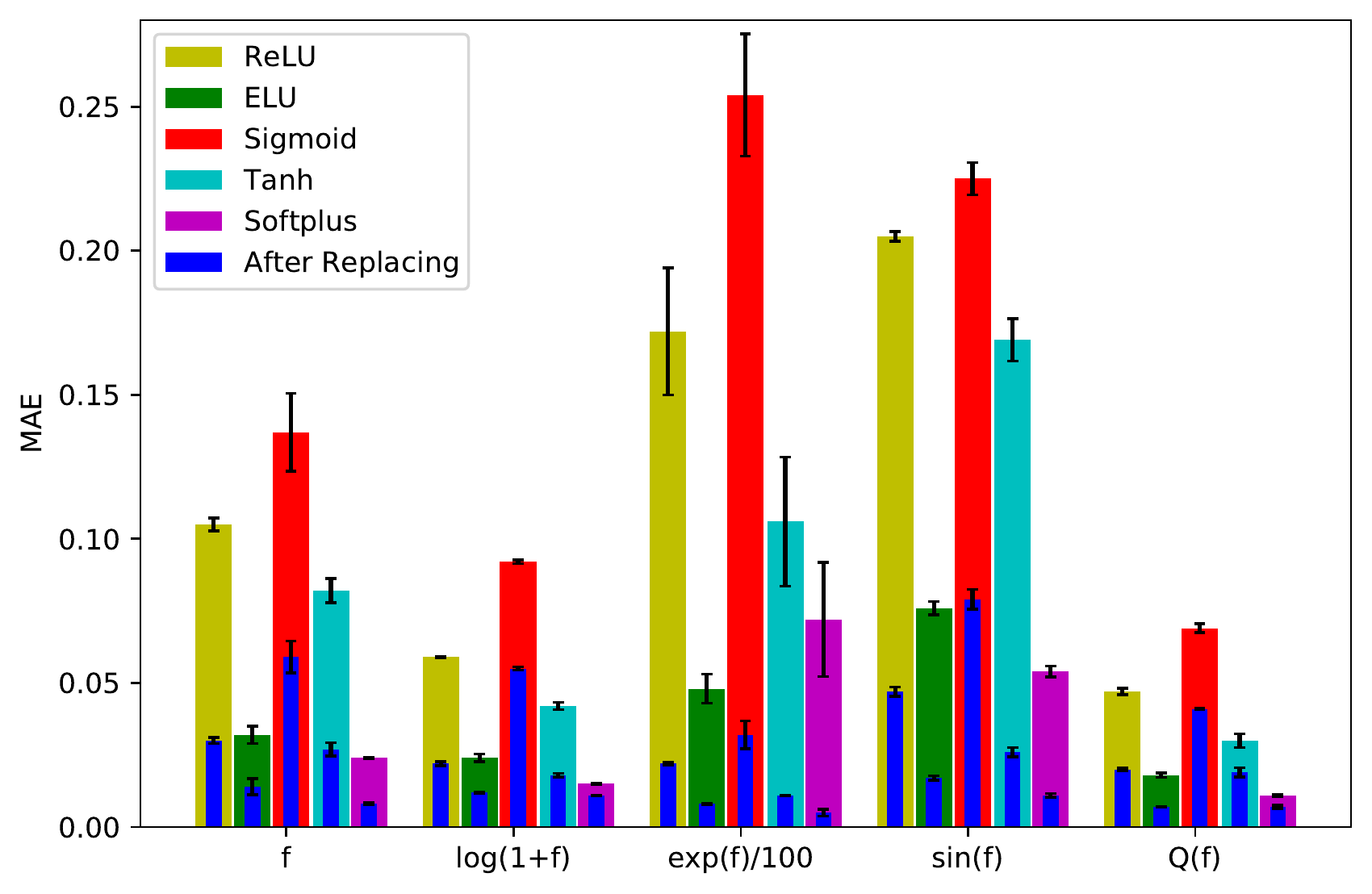}
\includegraphics[width=3in]{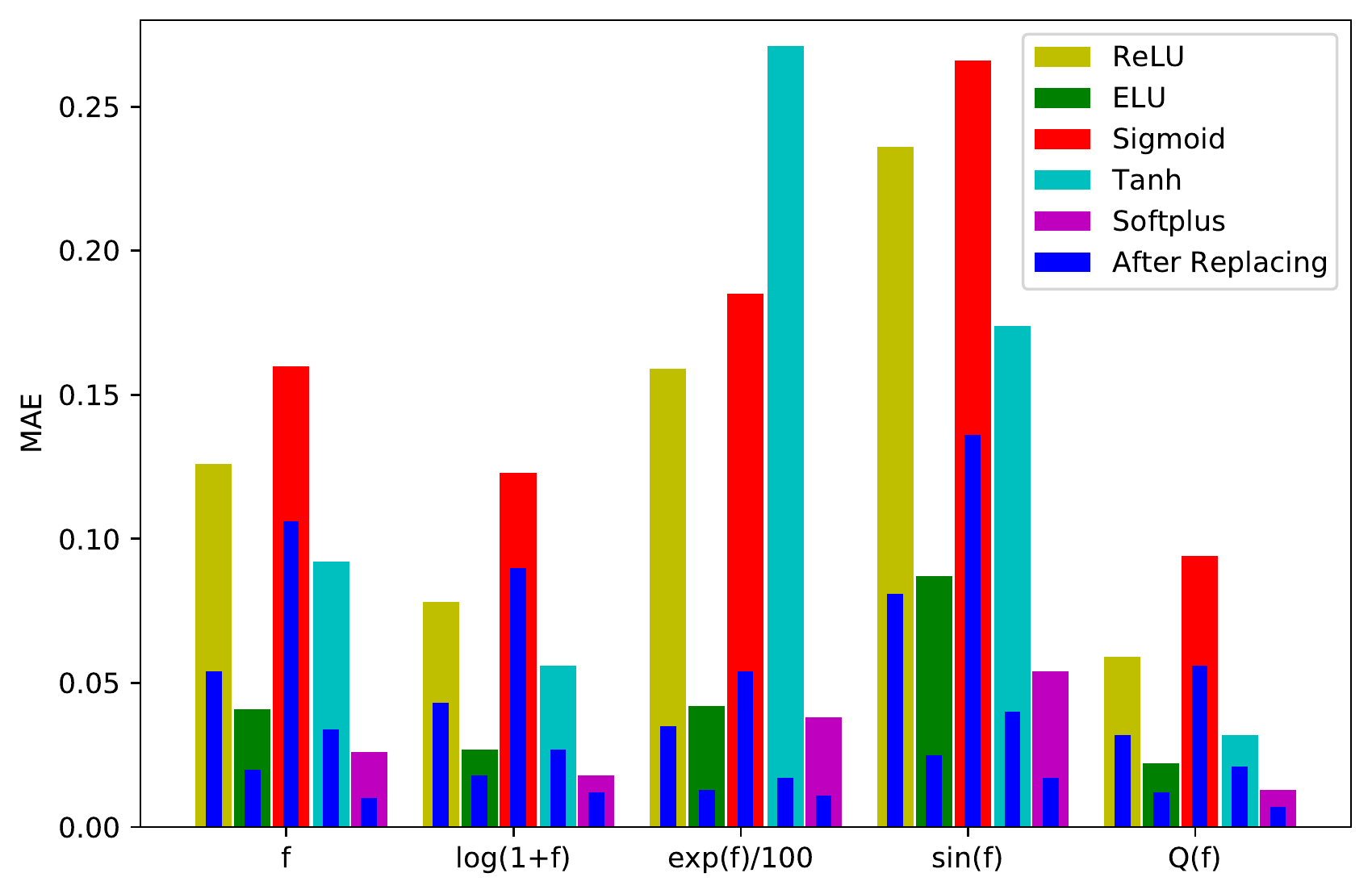}
    \caption{Effect of using Seagull activation functions on five regression tasks. Wider bars: MAE for the original model; Narrow blue bars: MAE of the model when the first activation function is replaced by Seagull. Left: Each experiment was performed 5 times independently. Right: with 5\% noise added in.\label{tab:example1}}
\end{figure}

We evaluated the performance of the network by Mean Absolute Error (MAE), calculated as follows:
\begin{equation}
MAE = \frac{1}{N}\sum_{j=1}^{N}|y_j-\hat{y_j}|%%
\end{equation}%%
where $N$ is the number of testing samples, $y_j$ and $\hat{y_j}$ are the ground truth and the prediction, respectively. The reason of first choosing MAE instead of MSE is that MAE is less related to the square function appeared in the new activation function in $\log(1+x^2)$.

The experiment was performed 5 times independently with different number random seeds. The results are reported in Figure 1-left. The results for different target functions are consistent.

In real-world applications, data always contains noise. To test the effectiveness of our method on noisy data, we added random noise of mean $0$ and $5\%$ standard deviation onto the training labels. The results in Figure 1-right demonstrate the robustness and effectiveness of the even activation function.

The improvement can also be visualized by the learning cure. To demonstrate this, we consider the $3\times3$ determinant formed by 9 variables $x_1,x_2,\ldots, x_9$. If we let $f_1$ be the cosine of the determinant. The function $f_1$ is partially exchangeable. To test the effectiveness of substituting with Seagull, we generated 50000 training vectors with each coordinate randomly generated from centered Gaussian distribution with variance 1, and in the same way independently generated 2000 test vectors. After experimenting and tuning the parameters, we built a fully-connected 8-layer neural network with 100 neurons in all middle layers, interlaced by two Batch Normalization layers and five Dropout layers of dropout rate 0.5, with Sigmoid as activation function. We training the model for 500 epochs and batch size 100 using RMSprop with learning rate 0.003 that is halved every 100 epochs to minimize the MSE. Then, for each model, we replaced the Sigmoid by Seagull in the first layer, leaving all the hyperparameters unchanged. Retraining the models from the beginning. The result is reflected in Figure 2.

\begin{figure}
\centering
\includegraphics[width=4in]{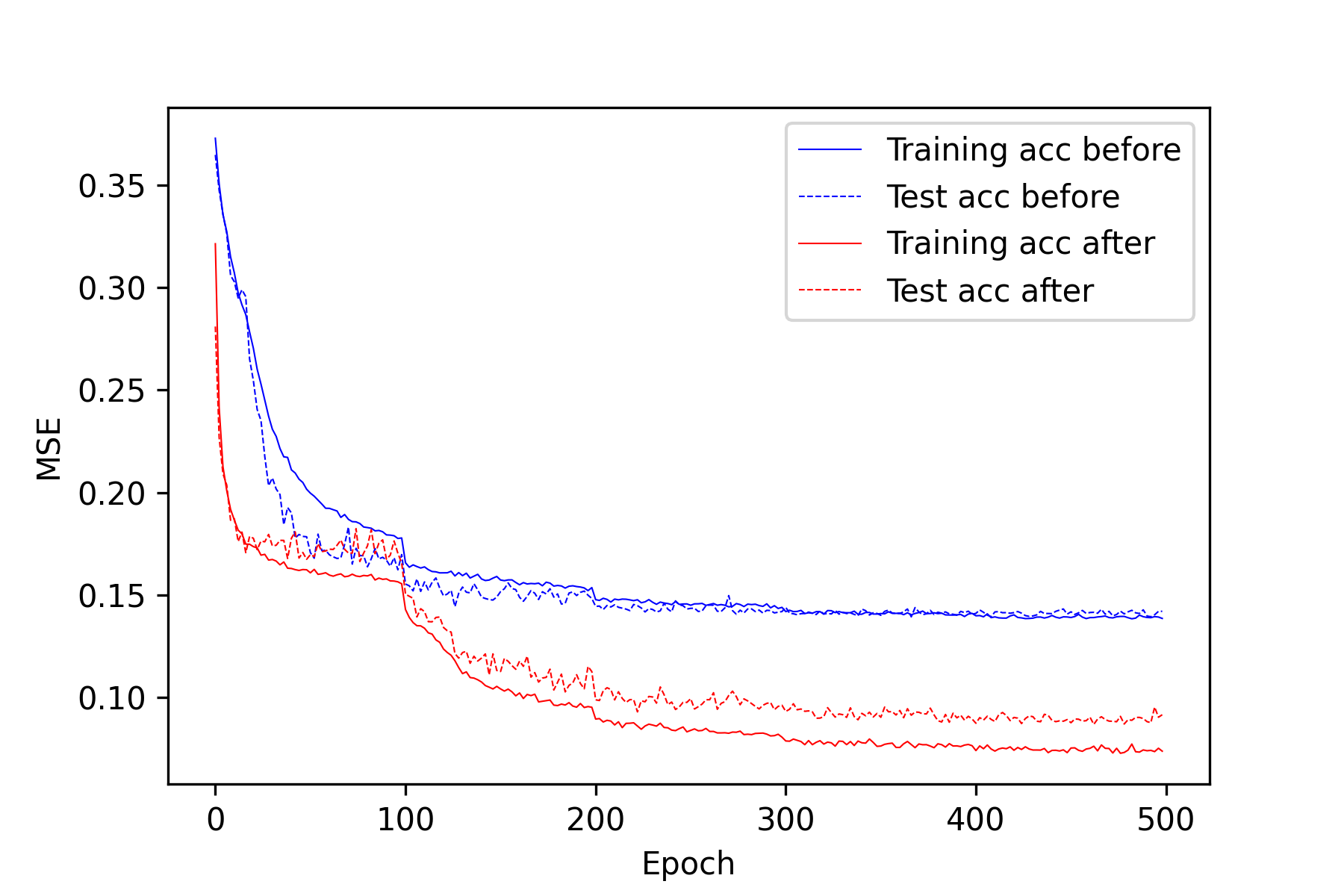}
\caption{Comparison of the learning curves.}
\end{figure}

\subsection*{Experiment on varying degree of exchangeability}
The function $f$ used in the previous two experiments are of strong exchangeability. Indeed, the function is obviously invariant when the three groups of variables $(x_1,x_2,x_3)$, $(x_4,x_5,x_6)$ and $(x_7,x_8,x_9)$ are permuted. Less obvious is that the function is also invariant when the three groups of variables $(x_1,x_4,x_7)$, $(x_2,x_5,x_8)$ and $(x_3,x_6, x_9)$ are permuted. There are total 12 different permutations under which the function $f$ is invariant.

To see the effect of even Seagull on target functions with varying degree of exchangeability, we also consider the following target function. To minimize the artifacts caused by dimensions, we also consider a 9-dimensional function.
Let $\phi$ the solid angle formed by the three vectors $u=(x_1,x_2,x_3)$, $v=(x_4,x_5, x_6)$, and $w=(x_7, x_8, x_9)$ on the unit sphere. The target function $\phi$ has a closed formula
\begin{align*}
\phi(x_1,x_2,x_3,x_4,x_5,x_6,x_7,x_8,x_9)=\phi(u,v,w) = \left\{\begin{array}{ll}2\tan^{-1} z& z\ge 0\\\pi+2\tan^{-1} z& z<0\end{array}\right.
\end{align*}
where $$z=\frac{|(u\times v)\cdot w|}{1+u\cdot v+v\cdot w+w\cdot u}.$$
Note that $\phi$ is intrinsically different from $f$, and the inner function $z$ even has singularity on the surface $u\cdot v+v\cdot w+w\cdot u =1$.
The function $\phi$ are invariant when the three groups of variables $(x_1,x_2,x_3)$, $(x_4,x_5,x_6)$ and $(x_7,x_8,x_9)$ are permuted, but not invariant when the three groups of variables $(x_1,x_4,x_7)$, $(x_2,x_5,x_8)$ and $(x_3,x_6, x_9)$ are permuted. There are 6 different permutations under which $\phi$ is invariant.

\begin{figure}[h]
\centering
    \includegraphics[width=6in]{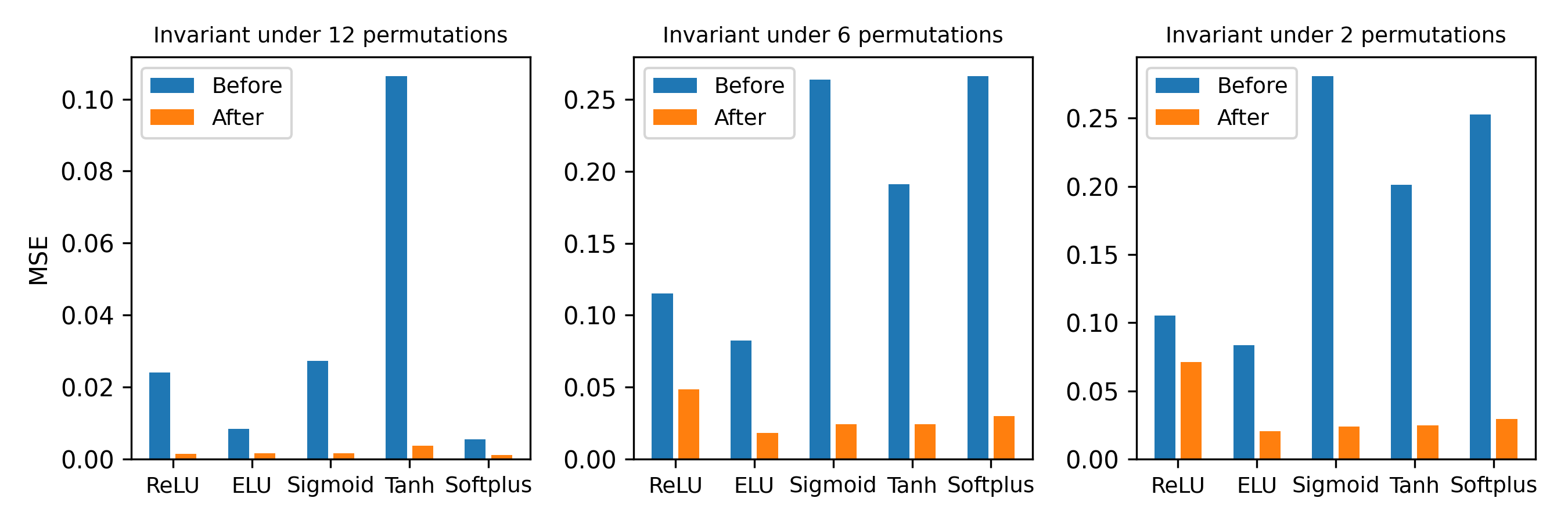}
    \caption{Effect of using Seagull activation functions on some 9 dimensional target functions with varying degree of exchangeability.}
\end{figure}
Furthermore, if we define $\psi(x_1,x_2,\ldots, x_9)=\phi(x_1,x_2,\ldots, x_8, -x_9)$. Then $\psi$ is invariant only when $(x_1,x_2,x_3)$ and $(x_4,x_5,x_6)$ are exchanged. Thus, $\psi$  has the weakest exchangeability (2 permutations),  while $f$ has the strongest exchangeability (12 permutations) among the three functions $f$, $\phi$ and $\psi$.

We used the same approach to test the effectiveness of substituting the original activation in the first layer by Seagull. For each function, we built five models. The number of layers and the number of neurons in each layer are the same, but the activation functions are different (ReLU, ELU, Sigmoid, Tanh, Softplus). We used MSE as loss function and experimented with three different optimizers (SGD, Adam and RMSprop), each with three choices of learning rate. The results in Figure 3 shows that while improvement seems to be positively correlated to the degree of exchangeability, for low-dimensional case, the improvement can be significant when there is only weak exchangeability.

Without the exchangeability, a substitution with Seagull is not expected to have any benefit and could even get worse. To verify the statement, we used model and function associated with Figure 2, but replacing the cosine function with the sine function. The new function is no longer partially exchangeable, and the substitution did not lead to any noticeable improved. Indeed, in for some activation functions, we even got slightly worse results, partly because the hyperparameters were not turned after the substitution.

\subsection*{Experiment on applying Seagull on different layers}
The layer on which the Seagull function is applied matters. The most effect way is to put it in the layer where the exchangeable variables are connected for the first time. Figure 3-left shows that the effectiveness quickly decreases. If the activation function is put at the last layer, the effectiveness could vanish. In one case, it even get worse performance. We believe this is partly due to the fact that the parameters were tuned for the original activation function, but not adjusted after the substitution.

\begin{figure}[h]
\centering
    \includegraphics[width=3in]{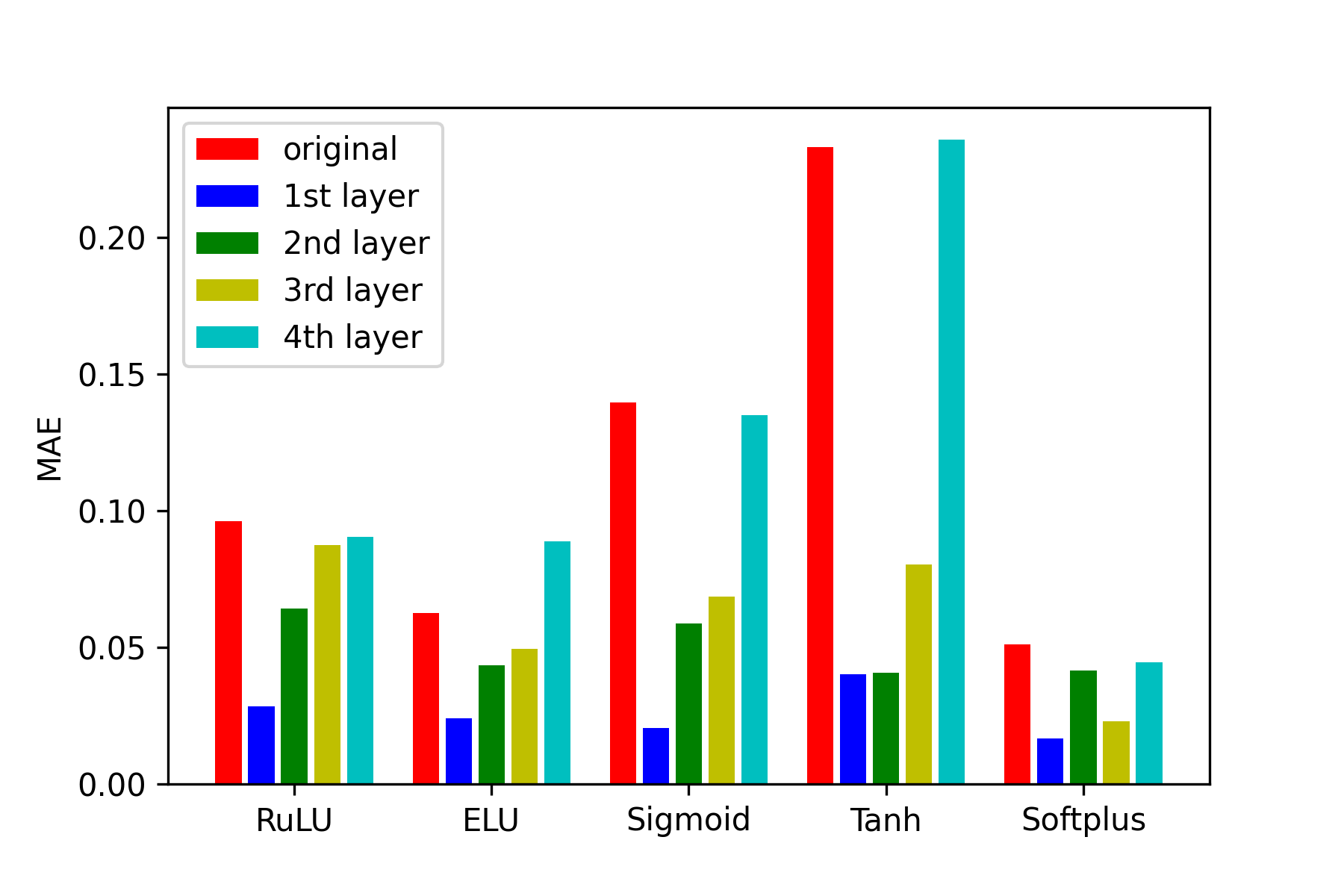} % Reduce the figure size so that it is slightly narrower than the column. Don't use precise values for figure width.This setup will avoid overfull boxes.
\includegraphics[width=3in]{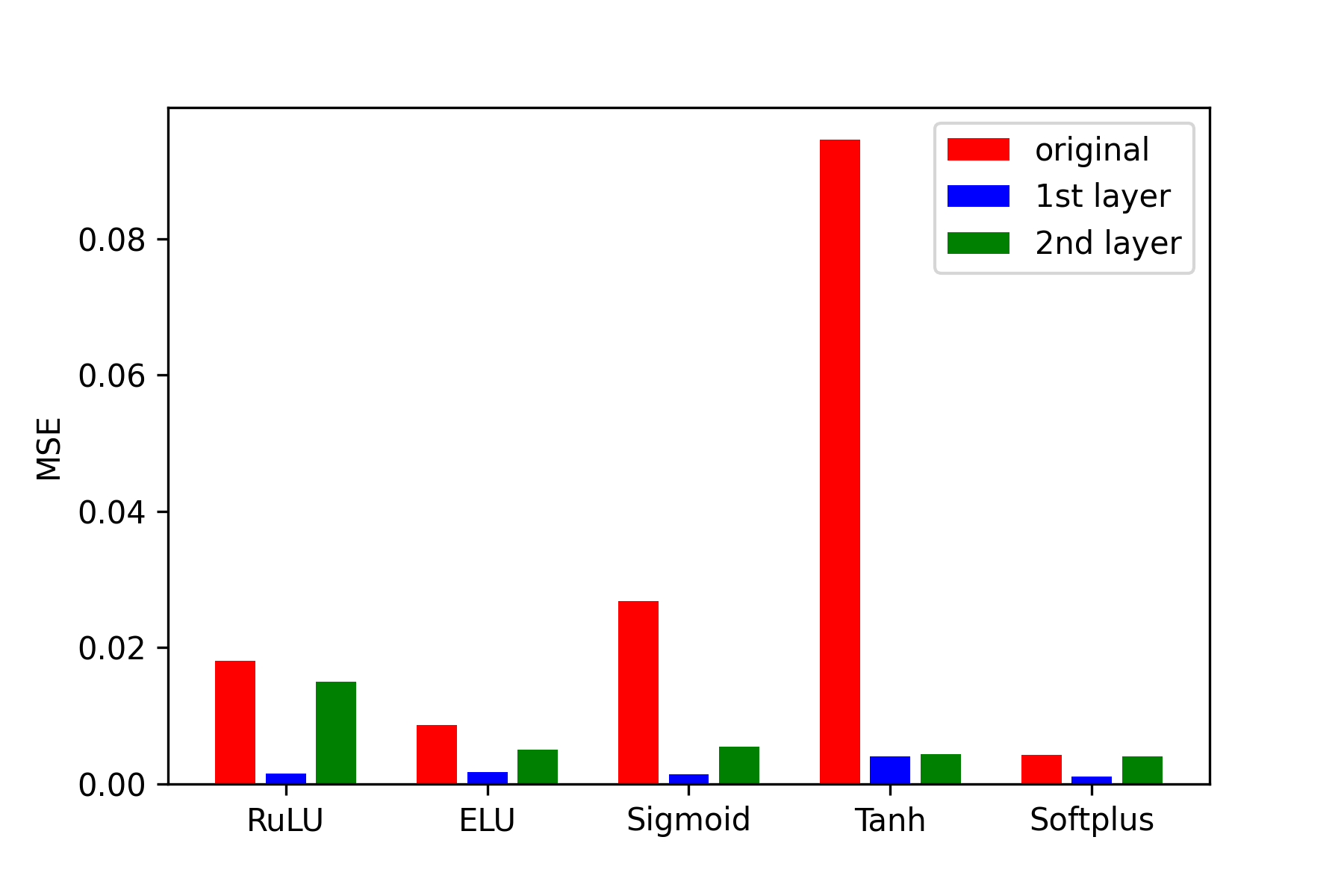} % Reduce the figure size so that it is slightly narrower than the column. Don't use precise values for figure width.This setup will avoid
    \caption{Effect of using Seagull activation functions on different layers. Left: MAE loss. Right: MSE loss.}
\end{figure}

The improvement is not an artifact of the loss function MAE. Figure 3-right shows that replacing MAE loss with the Mean Square Error (MSE), the improvement is also significant.

While it is possible that a more carefully designed neural network with longer training and finer tuning of the learning rates may produce predictions with smaller errors, the effect of using Seagull activation function in the first layer is evident from all examples.

\subsection*{Experiment on dimensions}

The effect of partial exchangeability will decrease as the dimension of the data increases. Indeed, in high-dimensional case, there are many combinations of the variables, and the ones satisfying partial exchangeability condition are only a tiny fraction. Thus, it is expected that the improvement of using a Seagull activation will decrease as the dimension increases.
The results in the previous experiments are for relatively low dimensions (9 dimensions). To see the dependence on the dimension, we experimented with 16, 25, 36, and 64 dimensional target functions. The improvement quickly decreases as dimension increases.  For example, for 25-dimensional functions, we experimented on the volume of 5-dimensional simplex generated by the origin and five 5-dimensional points whose coordinated are generated by normally distributed with mean 0 and variance 1. The function is invariant under 10 different permutations. We used the same experimental strategy as in the previous examples, but the improvement reduced to 12.7\% on average. When we increase the dimension to 64, we observed no improvement without tuning the parameters after substituting the activation function.

\begin{figure}[h]
\centering
    \includegraphics[width=4in]{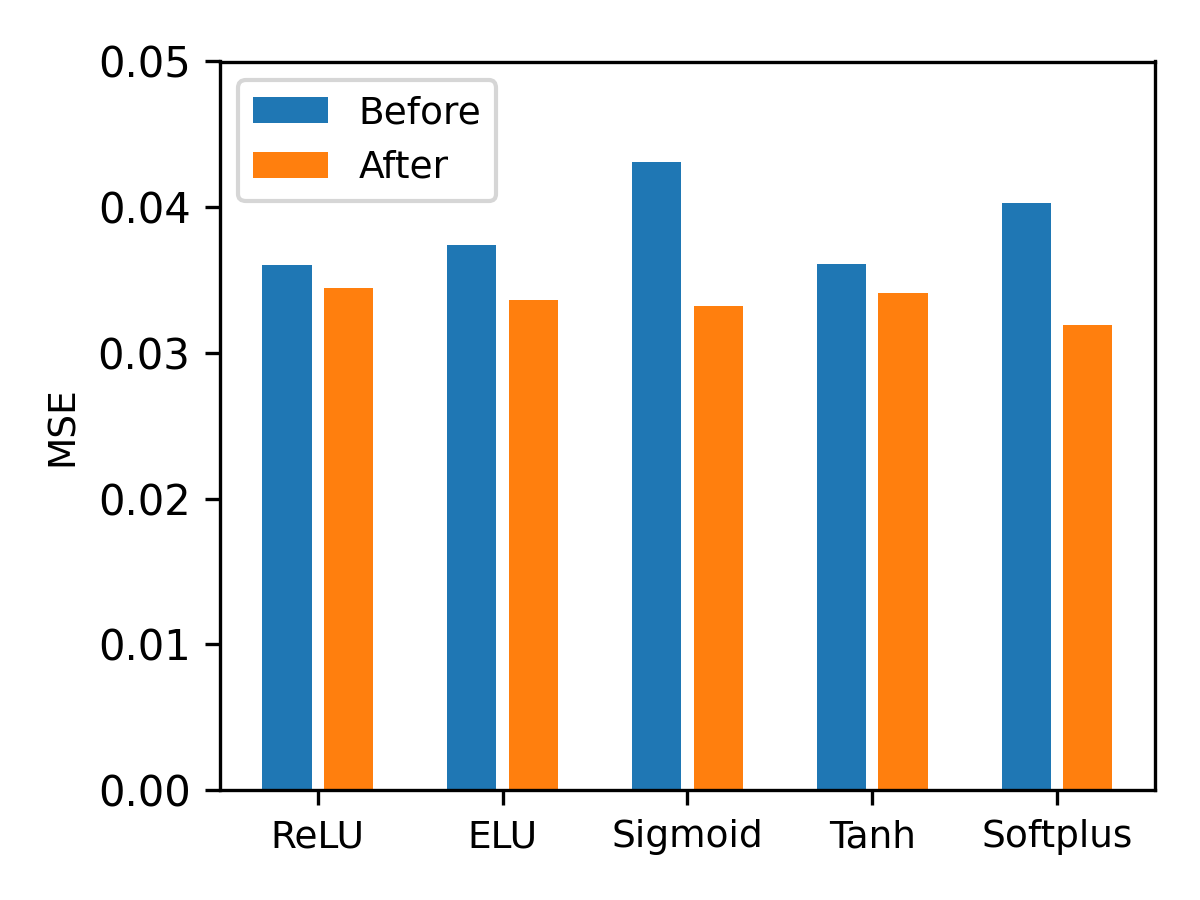}
    \caption{Effect of using Seagull activation functions on a 25-dimensional target function.}
\end{figure}

To see if improvement can still be made for very high-dimensional partially exchangeable target functions when parameters are tuned after the substitution, we performed experiments on CIFAR 10, a popular data set that consists of $60000$ $32\times32$ color images in $10$ classes, with $6000$ images per class. More details about the CIFAR-10 data set could be found in \cite{krizhevsky2009learning}.)

The dimension of the data is $3072$. There is partial exchangeability due to the left-right flipping invariance. One may also argue that by shuffling RGB channels, the classification is likely to be invariant. Although, the partial exchangeability is very limited in such a high-dimensional data, we can still see the effectiveness of using a Seagull activation function.

The following DCNNs were tested: Regnet \cite{radosavovic2020designing}, Resnet \cite{he2016deep}, VGG \cite{simonyan2014very}, and DPN \cite{chen2017dual}.
The details of each of these networks vary. In particular, we employ Regnet with 200 million flops (Regnet200MF), 50-layer Resnet (Resnet50), VGG net with 13 convolutional layers and 3 fully connected layers (VGG16), and DPN-26.

There are different ways to introduce an even activation function into the DCNNs.
Specifically, we replaced the ReLU at the output of each block for Regnet200MF, Resnet50, and VGG16.
Taking Resnet50 as an example, each block in the Resnet50 consists of a few weighted layers and a skip connection.
We replaced the last ReLU, which directly affects the output of the block, with Seagull activation function $\log(1+x^2)$ and left the rest of the activation functions unchanged.

For training setup, 50,000 images were used as training samples and 10,000 images for testing.
The total epoch number was limited as 350.
We selected the SGD optimizer with 0.9 momentum and set the learning rate as 0.1 for epoch $[0, 150)$, $0.01$ for epoch $[150, 250)$, and $0.001$ for epoch $[250, 350)$.
All training was started from scratch.
Considering the random initialization of the DCNNs, each training process was performed 5 times, and the average accuracy on testing data and standard deviation were summarized in Figure \ref{fig:SeagullvsReLU}.

\begin{figure}[h]
\centering
\includegraphics[width=3.8in]{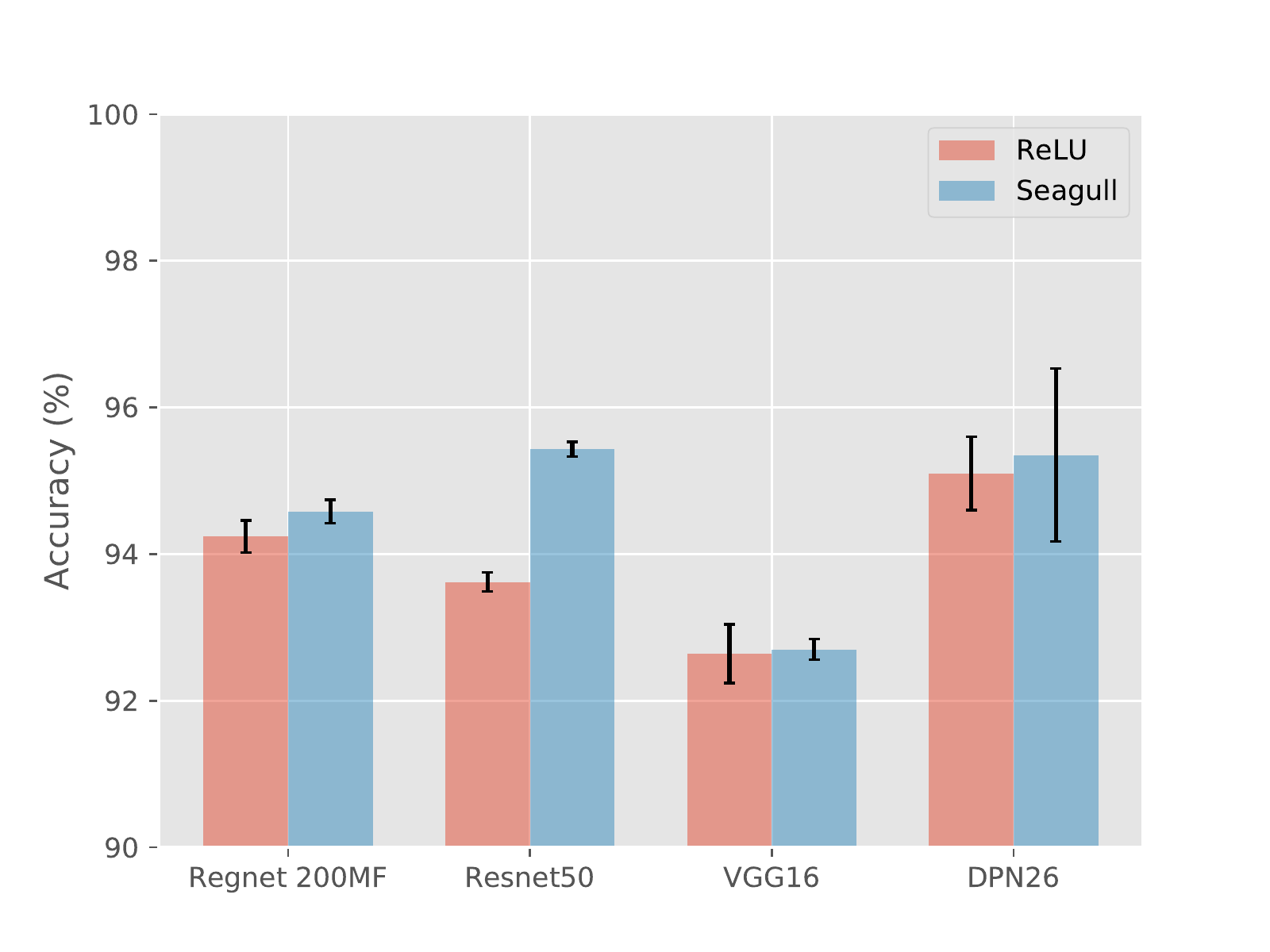} % Reduce the figure size so that it is slightly narrower than the column. Don't use precise values for figure width.This setup will avoid overfull boxes.
\caption{Performance of Seagull activation function and ReLU on CIFAR-10 under different neural network structures. The red bars reflect the average accuracy of ReLU and the blue bars for the Even activation function on CIFAR-10 data set, respectively. The black lines on the top of the bars present the standard deviation of the accuracy.}
\label{fig:SeagullvsReLU}
\end{figure}

The results in Figure \ref{fig:SeagullvsReLU} reveal some interesting information.
The Seagull function presents significant improvement compared to the ReLU function on all three DCNNs, despite some variation.
The feature maps in the network retain the partial exchangeability of the images.
Considering the DCNN structures, the convolution operation and activation functions leave the spatial relationship unchanged as well as the partial exchangeability.
According to Theorem \ref{th1}, the even activation function could capture this feature better compared to ReLU.
While one might argue the rotation and flip invariance could be achieved by data augmentation and filter rotation \cite{gao2017efficient}, using the proposed Seagull function would provide the same feature, as well as significant reduction of the network size and training time.

For DPN-26, a different strategy was selected.
First, we train 300 epochs on the original DPN-26 from scratch. Then we insert a fully-connected layer with an even activation function and trained the neural network with the same hyper parameters for 300 epochs.
The result shows noticeable improvements: For the network without using even activation functions, the average accuracy reached $95.10\%\pm 0.11\%$, whereas the network using an even activation function reached $95.35\% \pm 0.04\%$.
The $95.35\%$ accuracy also outperformed DPN-92, which consists of many more parameters.
This substantial improvement further demonstrates the effectiveness of using even activation functions for partially exchangeable target functions.

\section{Conclusion}
In this paper, we emphasized the importance of studying activation functions in neural networks.
We theoretically proved and experimentally validated on synthetic and real-world data sets that when the target function is partially exchangeable, using Seagull activation function in the layer when the exchangeably variables are connected for the first time can improve neural network performance.
Through a special yet commonly encountered case, these results demonstrate the great potential of customizing activation functions.

\paragraph{Acknowledgement} We would like to acknowledge the support
from the National Science Foundation grants OCA-1940270 and 2019609, 
and the National Institutes of Health grant P20GM104420.

\vskip 0.2in
\bibliographystyle{spmpsci}
\bibliography{Activation_reference}

\end{document}